\documentclass{article}
\usepackage{comment}

\usepackage{amsmath,amssymb}
\usepackage{subcaption}

\usepackage{tikz, pgfplots}
\pgfplotsset{compat = newest}

\usepackage[hidelinks]{hyperref}

\usepackage{amsthm}
\newtheorem{lemma}{Lemma}[section]
\newtheorem{proposition}[lemma]{Proposition}
\newtheorem{theorem}[lemma]{Theorem}
\theoremstyle{remark}
\newtheorem{example}[lemma]{Example}

\newtheorem*{densprop}{Density Property}

\newcommand{\GNNLang}{\text{MPLang}}
\newcommand{\GNN}{\text{MPNN}}
\newcommand{\ReLU}{\mathrm{ReLU}}
\newcommand{\identity}{\mathrm{id}}
\newcommand{\id}{\identity}
\newcommand{\GFMT}{\text{GFMT}}

\newcommand{\Rb}{\mathbb{R}}
\newcommand{\Ones}{1}
\newcommand{\Diamon}{\mathop\lozenge}

\newcommand{\G}{\mathbf G}
\newcommand{\Grs}{\mathbb{G}}
\newcommand{\Neigh}[2]{N(#1)(#2)}
\newcommand{\dist}{\rho_{\Grs_p,X}}

\newcommand{\Feat}[2]{\mathit{Feat}(#1,#2)}
\newcommand{\FeatX}[3]{\mathit{Feat}(#1,#2,#3)}

\renewcommand{\subset}{\subseteq}

\title{On the expressive power of \\
message-passing neural networks as \\ global feature map transformers}

\author{Floris Geerts \\
\normalsize University of Antwerp \\
\normalsize floris.geerts@uantwerpen.be \and
Jasper Steegmans\thanks{Supported by the Special
Research Fund (BOF) of Hasselt University.} \\
\normalsize Hasselt University \\
\normalsize jasper.steegmans@uhasselt.be \and
Jan Van den Bussche \\
\normalsize Hasselt University \\
\normalsize jan.vandenbussche@uhasselt.be}

\date{}

\begin{document}

\maketitle

\begin{abstract}
	We investigate the power of message-passing neural networks
	(MPNNs)
	in their capacity to transform the numerical features
	stored in the nodes of their input graphs.  Our focus is on
	global expressive power, uniformly over all input graphs, or over
	graphs of bounded degree with features from a bounded domain.
	Accordingly, we introduce the notion of a global feature map
	transformer (GFMT)\@.  As a yardstick for expressiveness, we use a
	basic language for GFMTs, which we call MPLang.  Every MPNN can
	be expressed in MPLang, and our results clarify to which extent
	the converse inclusion holds.  We consider exact versus
	approximate expressiveness; the use of arbitrary activation
	functions; and the case where only the ReLU activation function
	is allowed.

\end{abstract}

% \keywords{Closure under concatenation\and Semiring provenance semantics
% 	for modal logic\and Query languages for numerical data}

\section{Introduction}

An important issue in machine learning is the choice of formalism
to represent the functions to be learned
\cite{aima-book,understanding-ml-book}.  For example,
feedforward neural networks with hidden layers are a popular
formalism for representing functions from $\Rb^n$ to $\Rb^p$.
When considering functions over graphs, graph neural networks
(GNNs) have come to the fore \cite{hamilton-grl}.  GNNs come in
many variants; in this paper, specifically, we will work with the
variant known as message-passing neural networks (MPNNs)
\cite{gilmer-mpnn}.

MPNNs compute numerical values on the nodes of an input graph,
where, initially, the nodes already store vectors of numerical
values, known as \emph{features}.  Such an assignment of features
to nodes may be referred to as a \emph{feature map} on the graph
\cite{grohe-logic-gnn}.  We can thus view an MPNN as representing a
function that maps a graph, together with a feature map, to a new
feature map on that graph.  We refer to such functions as
\emph{global feature map transformers (GFMTs)}.

Of course, MPNNs are not intended to be directly specified by
human designers, but rather to be learned automatically from
input--output examples.  Still, MPNNs do form a language for
GFMTs.  Thus the question naturally arises:
what is the expressive power of this language?

We believe GFMTs provide a suitable basis for investigating this
question rigorously. The G for `global' here is borrowed from the
terminology of \emph{global function} introduced by Gurevich
\cite{gurevich_algfeas,gurevich_challenge}.  Gurevich was
interested in defining functions in structures (over some fixed
vocabulary) \emph{uniformly}, over all input structures.
Likewise, here we are interested in expressing GFMTs uniformly
over \emph{all} input graphs.  We also consider infinite
subclasses of all graphs, notably, the class of all graphs with a
fixed bound on the degree.

As a concrete handle on our question about the expressive power
of MPNNs, in this paper we define the language MPLang.  This
language serves as a yardstick for expressing GFMTs, in analogy
to the way Codd's relational algebra serves as a yardstick for
relational database queries \cite{ahv_book}.  Expressions in
MPLang can define features built arbitrarily from the input
features using three basic operations also found in MPNNs:
\begin{enumerate}
	\item
		Summing a feature over all neighbors in the graph, which
		provides the message-passing aspect;
	\item
		Applying an activation
		function, which can be an arbitrary continuous function;
	\item
		Performing arbitrary affine transformations (built using
		constants, addition, and scalar multiplication).
\end{enumerate}
The difference between MPLang-expressions and MPNNs is that
the latter must apply the above three operations in a rigid
order, whereas the operations can be combined arbitrarily in
MPLang.  In particular, every MPNN is readily expressible in
MPLang.

Our research question can now be made concrete: is, conversely,
every GFMT expressible in MPLang also expressible by an
MPNN\@?  We offer the following answers.
\begin{enumerate}
	\item
		We begin by considering the case of the popular activation
		function $\ReLU$
		\cite{goodfellow-book,arora-relu}. In this case, we show that
		every MPLang expression can indeed be converted into an
		MPNN (Theorem~\ref{thmrelu}).
	\item
		When arbitrary activation functions are allowed, we show that
		Theorem~1 still holds in restriction to
		any class of graphs of bounded degree, equipped with
		features taken from a bounded domain (Theorem~\ref{thmcont}).
	\item
		Finally, when the MPNN is required to use the ReLU activation
		function, we show that every MPLang expression can still
		be \emph{approximated} by an MPNN; for this result we again
		restrict to graphs of bounded degree, and moreover to
		features taken from a compact domain (Theorem~\ref{thmapprox}).
\end{enumerate}

This paper is organized as follows. Section~\ref{secrel}
discusses related work. Section~\ref{secdefs} defines GFMTs, MPNNs and
MPLang formally. Sections \ref{secrelu}, \ref{secont} and
\ref{secapprox} develop our Theorems \ref{thmrelu}, \ref{thmcont}
and \ref{thmapprox}, respectively.
We conclude in Section~\ref{seconc}.

Certain concepts and arguments assume some
familiarity with real analysis \cite{rudin}.

\section{Related work} \label{secrel}

The expressive power of GNNs has received a great deal of
attention in recent years. A very nice introduction, highlighting
the connections with finite model theory and database theory, has
been given by Grohe \cite{grohe-logic-gnn}.

One important line of research is focused on characterizing the
distinguishing power (also called separating power) of GNNs, in
their many variants.  There, one is interested in the question:
given two graphs, when can they be distinguished by a GNN\@?
This question is closely related to strong methods for graph
isomorphism checking, and more specifically, the Weisfeiler-Leman
algorithm.  A recent overview has been given by Morris et
al.~\cite{morris-wl-survey}.

Another line of research has as goal to extend classical results
on the ``universality'' of neural networks \cite{pinkus-survey}
to graphs \cite{grohe-random-gnn,azizian}.  (There are close
connections between this line of research and the one just
mentioned on distinguishing power \cite{gr-gnn-tl}.)  These
results consider graphs with a fixed number $n$ of nodes;
functions on graphs are shown to be approximable by appropriate
variants of GNNs, which, however, may depend on $n$.

A notable exception is the work by Barcel\'o et
al.~\cite{barcelo-log-expr-gnn,barcelo-gnn-sigrec}, which
inspired our present work.  Barcel\'o et al.\ were the first to
consider expressiveness of GNNs uniformly over all graphs
(note, however, the earlier work of Hella et al.\
\cite{hella-distributed} on similar message-passing distributed
computation models).  Barcel\'o et al.\ focus on MPNNs, which
they fit in a more general framework named AC-GNNs, and they also
consider extensions of MPNNs.  They further focus on \emph{node
	classifiers}, which, in our terminology, are GFMTs where the
input and output features are boolean values.  Using the
truncated ReLU activation function, they show that MPNNs can
express every node classifiers expressible in \emph{graded modal
	logic} (the converse inclusion holds as well).

In a way, our work can be viewed as generalizing the boolean
setting considered by Barcel\'o et al.\ to the numerical setting.
Indeed, the language MPLang can be viewed as giving a numerical
semantics to positive modal logic without conjunction, following
the established methodology of semiring provenance semantics for
query languages \cite{provsemirings,graedel-prov-modal}, and
extending the logic with application of arbitrary activation
functions.  By focusing on boolean inputs and outputs, Barcel\'o
et al.\ are able to capture a stronger logic than our positive
modal logic, notably, by expressing negation and counting.

We note that MPLang is a sublanguage of
the Tensor Language defined recently by one of us and Reutter
\cite{gr-gnn-tl}. That language serves to unify several GNN variants and
clarify their separating power and universality (cf.~the first
two lines of research on GNN expressiveness mentioned above).

Finally, one can also take a matrix computation perspective, and
view a graph on $n$ nodes, together with a $d$-dimensional
feature map, as an $n\times n$ adjacency matrix, together with
$d$ column vectors of dimension $n$.  To express GFMTs, one may
then simply use a general matrix query language such as MATLANG
\cite{matlang_tods}.  Indeed, results on the distinguishing power
of MATLANG fragments \cite{floris_lagraphs} have been applied to
analyze the distinguishing power of GNN variants \cite{balcilar}.
Of course, the specific message-passing nature of computation
with MPNNs is largely lost when performing general computations
with the adjacency and feature matrices.

\section{Models and languages} \label{secdefs}

In this section, we recall preliminaries on graphs; introduce the
notion of global feature map transformer (GFMT);
formally recall message-passing neural networks and define their
semantics in terms of GFMTs; and define the language MPLang.

\subsection{Graphs and feature maps} \label{secprelim}

We define a \emph{graph} as a pair $G = (V,E)$, where $V$ is the
set of nodes and $E \subseteq V \times V $ is the edge relation. We
denote $V$ and $E$ of a particular graph $G$ as $V(G)$ and $E(G)$
respectively. By default, we assume graphs to be finite,
undirected, and without loops, so $E$ is symmetric and
antireflexive.  If $(v,u) \in E(G)$ then we call $u$ a neighbor
of $v$ in $G$. We denote the set of neighbors of $v$ in $G$ by
$\Neigh{G}{v}$.  The number of neighbors of a node is called the
degree of that node, and the degree of a graph is the maximum
degree of its nodes.  We use $\Grs$ to denote the set of all
graphs, and $\Grs_p$, for a natural number $p$, to denote the set
of all graphs with degree at most $p$.

For a natural number $d$, a \emph{$d$-dimensional feature map} on
a graph $G$ is a function $\chi: V(G) \rightarrow \Rb^{d}$,
mapping the nodes to \emph{feature vectors}.  We use
$\Feat{G}{d}$ to denote the set of all possible $d$-dimensional
feature maps on $G$. Similarly, for a subset $X$ of $\Rb^d$, we
write $\FeatX{G}{d}{X}$ for the set of all feature maps from
$\Feat{G}{d}$ whose image is contained in $X$.

\subsection{Global feature map transformers}

Let $d$ and $r$ be natural numbers. We define a
\emph{global feature map transformer (GFMT) of type $d
		\rightarrow r$}, to be a function $T: \Grs \rightarrow
	(\Feat{G}{d} \rightarrow \Feat{G}{r})$. Thus, if $G$ is a graph
and $\chi$ is a $d$-dimensional feature map on $G$, then
$T(G)(\chi)$ is an $r$-dimensional feature map on $G$. We call
$d$ and $r$ the \emph{input} and \emph{output arity} of $T$,
respectively.

\begin{example} \label{exgfmt}
	We give a few simple examples, just to fix the notion, all with
	output arity 1. (GFMTs with higher output arities, after all, are just
	tuples of GFMTs with output arity 1.)
	\begin{enumerate}
		\item
			The GFMT $T_1$ of type $2 \to 1$ that assigns to every
			node the average of its two feature values.
			Formally, $T_1(G)(\chi)(v) = (x+y)/2$, where
			$\chi(v)=(x,y)$.
		\item
			The GFMT $T_2$ defined like $T_1$, but taking the maximum
			instead of the average.
		\item
			The GFMT $T_3$ of type $1 \to 1$ that assigns to every node
			the maximum of the features of its neighbors.  Formally,
			$T_3(G)(\chi)(v) = \max \{\chi(u) \mid u \in \Neigh Gv\}$.
		\item
			The GFMT $T_4$ of type $1 \to 1$ that assigns to every node
			$v$ the sum, over all paths of length two from $v$,
			of the feature values of the end nodes of the paths.
			Formally, $$ T_4(G)(\chi)(v) = \sum_{(v,u) \in E(G)}
				\sum_{(u,w) \in E(G)} \chi(w). $$
	\end{enumerate}
\end{example}

\subsection{Operations on GFMTs}

If $T_1,\dots ,T_r$ are GFMTs of type $d \rightarrow 1$, then
the tuple $(T_1,\dots,T_r)$ defines a GFMT $T$ of type $d
	\rightarrow r$ in the obvious manner:
\begin{equation}
	T(G)(\chi)(v) := (T_1(G)(\chi)(v),\dots ,T_r(G)(\chi)(v))
	\label{eq:tuple_of_GFMTs}
\end{equation}
Conversely, it is clear that any $T$ of type $d \rightarrow r$
can be expressed as a tuple $(T_1,\dots,T_r)$ as above, where
$T_i(G)(\chi)(v)$ equals the $i$-th component in the tuple
$T(G)(\chi)(v)$.

Related to the above tupling operation is concatenation.  Let
$T_1$ and $T_2$ be GFMTs of type $d \to r_1$ and $d \to
	r_2$, respectively. Their \emph{concatenation} $T_1 \mid T_2$
is the GFMT $T$ of type $d \to r_1 + r_2$ defined by
$T(G)(\chi)(v) = T_1(G)(\chi)(v) \mid T_2(G)(\chi)(v))$, where
$\mid$ denotes concatenation of vectors.  Concatenation is
associative.  Thus, we could write
the previously defined $(T_1,\dots,T_r)$ also as $T_1 \mid
	\dots \mid T_r$.

We also define the \emph{parallel composition} $T_1 \parallel T_2$ of
two GFMTs $T_1$ and $T_2$, of type $d_1 \to r_1$ and $d_2 \to
	r_2$, respectively.  It is the GFMT $T$ of type $(d_1 + d_2)
	\to (r_1 + r_2)$ defined by $T(G)(\chi)(v) = T_1(G)(\chi_1)(v)
	\mid T_2(G)(\chi_2)(v)$, where $\chi_1$ ($\chi_2$) is the
feature map that assigns to any node $w$ the projection of
$\chi(w)$ to its first (last) $d_1$ ($d_2$) components.

In contrast, the \emph{sequential composition} $T_1;T_2$ of two
GFMTs $T_1$ and $T_2$, of type $d_1 \rightarrow d_2$ and $d_2
	\rightarrow d_3$ respectively, is the GFMT $T$ of type $d_1 \to
	d_3$ that maps every graph $G$ to $T_2(G) \circ T_1(G)$.  In
other words, $(T_1;T_2)(G)(\chi)(v) = T_2(G)(T_1(G)(\chi))(v)$.

Finally, for two GFMTS $T_1$ and $T_2$ of type $d \to r$, we
naturally define their sum $T_1 + T_2$ by $(T_1 +
	T_2)(G)(\chi)(v) := T_1(G)(\chi)(v) + T_2(G)(\chi)(v)$ (addition
of $r$-dimensional vectors). The difference $T_1-T_2$ is defined
similarly.

\newcommand{\Thalf}{T_{\rm half}}
\newcommand{\Tsum}{T_{\rm sum}}
\begin{example} \label{exops}
	Recall $T_1$ and $T_4$ from Example~\ref{exgfmt}, and consider the
	following simple GFMTs:
	\begin{itemize}
		\item
			For $j=1,2$, the GFMT $P_j$ of type $2\to 1$ defined by
			$P_j(G)(\chi)(v) = x_j$, where $\chi(v)=(x_1,x_2)$.
		\item
			The GFMT $\Thalf$ of type $1\to 1$ defined by
			$\Thalf(G)(\chi)(v) = \chi(v)/2$.
		\item
			The GFMT $\Tsum$ of type $1\to 1$ defined by
			$$ \Tsum(G)(\chi)(v) = \sum_{u \in \Neigh
					Gv} \chi(u). $$
	\end{itemize}
	Then $T_1$ equals $(P_1+P_2) ; \Thalf$, and $T_4$ equals
	$\Tsum ; \Tsum$.
\end{example}

\subsection{Message-passing neural networks} \label{secmpnn}

A \emph{message-passing neural network (MPNN)} consists of
layers.  Formally, let $d$ and $r$ be natural numbers.
An \emph{MPNN layer of type $d \rightarrow r$} is a 4-tuple
$L=(W_1, W_2,b,\sigma)$, where $\sigma: \Rb \rightarrow \Rb$ is a
continuous function, and $W_1$, $W_2$ and $b$ are real matrices
of dimensions $r \times d$, $r \times d$ and $r \times 1$,
respectively.  We call $\sigma$ the \emph{activation function} of
the layer; we also refer to $L$ as a \emph{$\sigma$-layer}.

An MPNN layer $L$ as above defines a $\GFMT$ of type $d
	\rightarrow r$ as follows:
\begin{equation}
	\label{eq-layer}
	L(G)(\chi)(v) := \sigma \bigl( W_1 \chi(v) + W_2 \sum_{u \in
		\Neigh{G}{v}} \chi(u) + b \bigr).
\end{equation}

In the above formula, feature vectors are used as \emph{column
	vectors}, i.e., $d \times 1$ matrices.
The matrix multiplications involving $W_1$ and $W_2$ then produce
$r \times 1$ matrices, i.e.,
$r$-dimensional feature vectors as desired. We see that matrix $W_1$
transforms the feature vector of the current node from a
$d$-dimensional vector to an $r$-dimensional vector. Matrix $W_2$
does a similar transformation but for the sum of the feature
vectors of the neighbors.  Vector $b$ serves as a \emph{bias}.
The application of $\sigma$ is performed component-wise on the
resulting vector.

We now define an MPNN as a finite, nonempty sequence
$L_1,\dots ,L_p$ of MPNN layers, such that the input arity of
each layer, except the first, equals the output arity of
the previous layer.  Such an MPNN naturally defines a GFMT that
is simply the sequential composition $L_1;\dots;L_p$ of its layers.
Thus, the input arity of the first layer serves as the input
arity, and the output arity of the last layer serves as the
output arity.

\begin{example} \label{exmpnn}
	Recall the ``rectified linear unit'' function $\ReLU : \Rb \to
		\Rb : z \mapsto \max(0,z)$.  Observe that $\max(x,y) =
		\ReLU(y-x) + x$, and also that $x = \ReLU(x) - \ReLU(-x)$.
	Hence, $T_2$ from Example~\ref{exgfmt} can be
	expressed by a two-layer MPNN, where the first layer $L_1$
	transforms input feature vectors $(x,y)$ to feature vectors
	$(y-x,x,-x)$ and then applies $\ReLU$, and the second layer
	$L_2$ transforms the feature vector $(a,b,c)$ produced by $L_1$
	to the final result $a+b-c$.  Formally, $L_1 = (A,0^{3\times
		2},0^{3\times 1},\ReLU)$, with $$ A = \begin{pmatrix} -1           & 1 \\
                \phantom{-}1 & 0 \\ -1 & 0\end{pmatrix}, $$ and $L_2 =
		((1,1,-1),(0,0,0),0,\id)$, with $\id$ the identity function.

	For another, simple, example, $\Tsum$ from Example~\ref{exops}
	is expressed by the single layer $(0,1,0,\id)$.
\end{example}

\paragraph{Same activation function}

If, for a particular MPNN, and an activation function $\sigma$,
all layers except the last one are $\sigma$-layers, and the last
layer is either also a $\sigma$-layer, or has the identity
function as activation function, we refer to the MPNN as a
$\sigma$-MPNN.  Thus, the two MPNNs in the above example are
$\ReLU$-MPNNs.

\subsection{MPLang}

We introduce a basic language for expressing GFMTs.
The syntax of expressions $e$ in MPLang is given by the following grammar:
$$ e ::= \Ones \mid P_i \mid a e \mid e + e \mid f(e) \mid \Diamon e $$
where $i$ is a non-zero natural number, $a \in \Rb$ is a
constant, and $f:\Rb \to \Rb$ is continuous.

An expression $e$ is called \emph{appropriate for input arity
	$d$} if all subexpressions of $e$ of the form $P_i$ satisfy $1
	\leq i \leq d$.  In this case, $e$ defines a GFMT of type $d \to
	1$, as follows:
\begin{itemize}
	\item if $e = \Ones$, then $e(G)(\chi)(v) := 1$
	\item if $e = P_i$, then $e(G)(\chi)(v) :=$ the $i$-th component of $\chi(v)$
	\item if $e = a e_1$, then $e(G)(\chi)(v) := a e_1(G)(\chi)(v)$
	\item if $e = e_1 + e_2$, then $e(G)(\chi)(v) := e_1(G)(\chi)(v) + e_2(G)(\chi)(v)$
	\item if $e = f(e_1)$, then $e(G)(\chi)(v) := f(e_1(G)(\chi)(v))$
	\item if $e = \Diamon e_1$, then $e(G)(\chi)(v) := \sum_{u \in \Neigh{G}{v}}e_1(G)(\chi)(u)$
\end{itemize}

To express higher output arities, we agree that a GFMT $T$
of type $d \to r$ is expressible in MPLang if there exists a
tuple $(e_1, \dots ,e_{r})$ of expressions that
defines $T$ in the sense of Equation~\ref{eq:tuple_of_GFMTs}.
We further agree:
\begin{itemize}
	\item
		The constant $a$ will be used as a shorthand for the
		expression $a \Ones$.
	\item
		For any fixed function $f$, we denote by $f$-MPLang
		the language fragment of MPLang where
		all function applications apply $f$.
\end{itemize}

\begin{example} \label{exmplang}
	Continuing Example~\ref{exmpnn}, also $T_2$ and $\Tsum$
	can be expressed in
	MPLang, namely, $T_2$ as $\ReLU(P_2-P_1)+P_1$, and
	$\Tsum$ as $\Diamon P_1$.
\end{example}

\subsection{Equivalence} \label{secequiv}

Let $T_1$ and $T_2$ be MPNNs, or tuples of MPLang expressions, of
the same type $d\to r$.
\begin{itemize}
	\item
		We say that $T_1$ and $T_2$ are
		\emph{equivalent} if they express the same GFMT.
	\item
		For a class
		$\G$ of graphs and a subset $X$ of $\Rb^d$, we say that
		$T_1$ and $T_2$ are \emph{equivalent over $\G$ and $X$} if the
		GFMTs expressed by $T_1$ and $T_2$ are equal on every graph $G$
		in $\G$ and every $\chi \in \FeatX GdX$ (see
		Section~\ref{secprelim}).
\end{itemize}

Example~\ref{exmplang} illustrates the following general
observation:

\begin{proposition} \label{propeasy}
	For every MPNN $T$ there is an equivalent tuple of MPLang-expressions
	that apply, in function applications, only
	activation functions used in $T$.
\end{proposition}
\begin{proof}
	Since we can always substitute subexpressions of the form $P_i$ by
	more complex expressions, MPLang is certainly closed under
	sequential composition. It thus suffices to verify that single
	MPNN layers $L$, or even the separate ingredients of a layer
	are expressible in MPLang. For each output
	component of $L$ we devise a separate MPLang expression.
	We create an expression for the $j$-th component.
	Inspecting Equation~\ref{eq-layer}, we must argue for linear
	transformation; summation over neighbors; addition of a
	constant (component from the bias vector); and application of
	an activation function.

	Linear transformation appears when multiplying an $r \times d$ matrix $W$
	with a $d$-dimensional vector $\chi(v)$. Let $w_{k}$ be
	the value of $W_1$ at the $j$-th row and $k$-th column.
	The translation of the $j$-th component of $W\chi{v}$ is $w_{1}P_1 + \dots + w_{d}P_d$.

	The addition of the bias vector $b$ for the $j$-th component is
	the addition of $j$-th component $b_j$ to an expression $e$.
	The translation is then $e + b_j$.

	Summation over neighbors is a component-wise summation. The translation
	of the summation over the $j$-th component of the feature vectors of the neighbors
	of the current node is $\Diamon(P_j)$.

	Application of an activation function is provided by
	function application in MPLang.
\end{proof}

\section{From MPLang to MPNN under ReLU} \label{secrelu}

In Proposition~\ref{propeasy} we observed that MPLang readily
provides all the operators that are implicitly present in MPNNs.
MPLang, however, allows these operators to be combined
arbitrarily in expressions, whereas MPNNs have a more rigid
architecture.  Nevertheless, at least under the ReLU activation
function, we have the following strong result:

\begin{theorem}
	\label{thmrelu}
	Every GFMT expressible in ReLU-MPLang is also expressible as a
	ReLU-MPNN.
\end{theorem}

Crucial to proving results of this kind will be that the MPNN
architecture allows the construction of concatenations of MPNNs.
We begin by noting:

\begin{lemma} \label{lem-layer-concat}
	Let $\sigma$ be an activation function. The class of GFMTs
	expressible as a single $\sigma$-MPNN layer is
	closed under concatenation and under parallel composition.
\end{lemma}
\begin{proof}
	For parallel composition, we construct block-diagonal
	matrices from the matrices provided by the two layers.
	Let $L=(W_{1L}, W_{2L}, b_L, \sigma)$ and $K=(W_{1K}, W_{2K}, b_K, \sigma)$
	be two layers of type $d_L \to r_l$ and $d_K \to r_K$ respectively.
	The layer $J=(W_1,W_2,b,\sigma)$ expresses $L \parallel K$, with $W_1$ equal to
	\begin{math}
		\begin{pmatrix}
			W_{1L} & \vline & 0      \\
			\hline
			0      & \vline & W_{1K}
		\end{pmatrix}
	\end{math},
	and $W_2$ constructed similarly using $W_{2L}$ and $W_{2K}$. The vector $b$ is $b_K \mid b_L$.

	For concatenation, we can simply stack the matrices vertically.
	More formally, assume $d_K = d_L$, then $J$ expresses $L \mid K$, if $W_1$ is equal to
	\begin{math}
		\begin{pmatrix}
			W_{1L} \\
			\hline
			W_{1K}
		\end{pmatrix}
	\end{math}
	and $W_2$ is constructed similarly, using $W_{2L}$ and $W_{2K}$. The vector $b$ is again $b_K \mid b_L$.
\end{proof}

For $\sigma = \ReLU$,
we can extend the above Lemma to multi-layer MPNNs:

\begin{lemma} \label{lem-relu-concat}
	ReLU-MPNNs are closed under concatenation.
\end{lemma}
\begin{proof}
	Let $L$ and $K$ be two ReLU-MPNNs.
	Since ReLU is idempotent, every $n$-layer ReLU-MPNN is
	equivalent to an $n+1$-layer ReLU-MPNN\@.  Hence we may assume
	that $L=L_1;\dots,L_n$ and $K=K_1;\dots;K_n$ have the same
	number of layers.
	Now $L \mid K = (L_1 \mid K_1); (L_2
		\parallel K_2); \dots ; (L_n \parallel K_n)$ if $n \geq 2$; if
	$n=1$, clearly $L \mid K = L_1 \mid K_1$.
	Hence, the claim follows from Lemma~\ref{lem-layer-concat}.
\end{proof}

Note that a ReLU-MPNN layer can only output positive numeric
values, since the result of ReLU is always positive.  This
explains why we must allow the identity function (id) in the last
layer of a ReLU-MPNN (see the end of Section~\ref{secmpnn}).
Moreover, we can simulate intermediate id-layers in a ReLU-MPNN,
thanks to the identity $x = \ReLU(x)-\ReLU(-x)$.  Specifically,
we have:

\begin{lemma} \label{lem-idlayer}
	Let $L$ be an id-layer and let $K$ be a $\sigma$-layer. Then
	there exists a ReLU-layer $L'$ and a $\sigma$-layer $K'$ such
	that $L;K$ is equivalent to $L';K'$.
\end{lemma}
\begin{proof}
	Let $L=(W_1,W_2,b,\id)$.  We put $$ L' = (W_1,W_2,b,\ReLU) \mid
		(-W_1,-W_2,-b,\ReLU) $$ which corresponds to a ReLU-layer by
	Lemma~\ref{lem-layer-concat}.  Let $K=(A,B,c,\sigma)$.
	Consider the block matrices $A'=(A|{-A})$ and $B'=(B|{-B})$
	(single-row block matrices, with two matrices stacked
	horizontally, not vertically).  Now for $K'$ we use $(A',B',c,\sigma)$.
\end{proof}

We now ready to prove Theorem~\ref{thmrelu}.
By Lemma~\ref{lem-relu-concat}, it suffices to focus on MPLang
expressions, i.e., GFMTs of output arity one.
So, our task is to construct, for every expression $e$ in
ReLU-MPLang, an equivalent ReLU-MPNN $E$.  However,
by Lemma~\ref{lem-idlayer}, we are free to use intermediate
id-layers in the construction of $E$.  We proceed by induction on
the structure of $e$. Consider the base cases where $e$ is of the form
$1$ and $P_i$ and assume $e$ is appropriate for input arity $d$.
\begin{itemize}
	\item If $e$ is of the form $1$, we set $E = (\vec{0},\vec{0},1,\identity)$
		with $\vec{0} = 0^{1 \times d}$.
	\item If $e$ is of the form $P_i$, we set $L = (W_1,\vec{0},0,\identity)$
		with $\vec{0} = 0^{1 \times d}$ and $W_1$ the $i$-th
		canonical basis vector of dimension $d$, i.e.,
		$W_1= (0, \dots, 0, 1, 0, \dots, 0)$ with $1$ in
		the $i$-th position.
\end{itemize}
Consider the inductive cases where $e$ is of one of the forms
$ae_1$, $e_1 + e_2$, $f(e_1)$ (with $f=\ReLU$), or $\Diamon e_1$.
By induction, we have MPNNs $E_1$ and $E_2$ for $e_1$ and $e_2$.
\begin{itemize}
	\item
		If $e$ is of the form $ae_1$, we set $E =
			E_1 ; (a,0,0,\id)$.
	\item
		If $e$ is of the form $e_1 + e_2$, we set $E= (E_1
			\mid E_2) ; ((1,1),(0,0),0,\id)$. Here, $E_1 \mid E_2$
		corresponds to a ReLU-MPNN by Lemma~\ref{lem-relu-concat}.
	\item
		If $e$ is of the form $f(e_1)$, we set $E= E_1 ; (1,0,0,f)$.
	\item
		If $e$ is of the form $\Diamon e_1$, we set $E=E_1;
			(0,1,0,\id)$.
\end{itemize}

\section{Arbitrary activation functions} \label{secont}

Theorem~\ref{thmrelu} only supports the ReLU function in MPLang
expressions. On the other hand, the equivalent MPNN then only
uses ReLU as well. If we allow arbitrary activation functions in
MPNNs, can they then simulate also MPLang expressions that apply
arbitrary functions? We can answer this question affirmatively,
under the assumption that graphs have bounded degree and feature
vectors come from a bounded domain.

\begin{theorem}
	\label{thmcont}
	Let $p$ and $d$ be natural numbers, let $\Grs_p$ be the class of
	graphs of degree at most $p$, and let $X \subset
		\Rb^{d}$ be bounded.  For every GFMT $T$ expressible in MPLang
	there exists an MPNN that is equivalent to $T$ over $\Grs_p$ and
	$X$.
\end{theorem}

The above theorem can be proven exactly as Theorem~\ref{thmrelu},
once we can deal with the concatenation of two MPNN layers with
possibly different activation functions.  The following result
addresses this task:

\begin{figure}
	\centering
	\begin{subfigure}{0.47\textwidth}
		\centering
		\begin{tikzpicture}
			\begin{axis}[
					xlabel = \(x\),
					xmin=-1.1, xmax=3.2,
					ylabel = \(\sigma_1\),
					ymin=-1.1, ymax=1.2,
					axis lines = middle,
					width =\textwidth,
					height = 5cm
				]
				\addplot [
					domain = -1:3,
					color = red
				] {
					tanh(x)
				};
			\end{axis}
		\end{tikzpicture}
	\end{subfigure}
	\quad
	\begin{subfigure}{0.47\textwidth}
		\centering
		\begin{tikzpicture}
			\begin{axis}[
					xlabel = \(x\),
					xmin=-2.1, xmax=1.2,
					ylabel = \(\sigma_2\),
					ymin=-2.1, ymax=1.2,
					axis lines = middle,
					width = \textwidth,
					height = 5cm
				]
				\addplot [
					domain = -2:1,
					color = blue
				] {
					x
				};
			\end{axis}
		\end{tikzpicture}
	\end{subfigure}
	\begin{subfigure}{0.95\textwidth}
		\centering
		\begin{tikzpicture}
			\begin{axis}[
					xlabel = \(x\),
					xmin=-5.1, xmax=4.1,
					ylabel = \(\sigma'\),
					ymin=-2.1, ymax=1.2,
					axis lines = middle,
					width = \textwidth,
					height = 5cm
				]
				\addplot [
					domain = -5:-1,
					color = red
				] {
					tanh(x+3+1)
				};
				\addplot [
					domain = 1:4,
					color = blue
				] {
					x - 2 - 1
				};
				\addplot [domain = -1:1] {
					-1.5*x - 0.5
				};
			\end{axis}
		\end{tikzpicture}
	\end{subfigure}
	\caption{Illustration of the proof of
		Lemma~\ref{lem-conc-cont}.}
	\label{fig:activation_function_merging}
\end{figure}
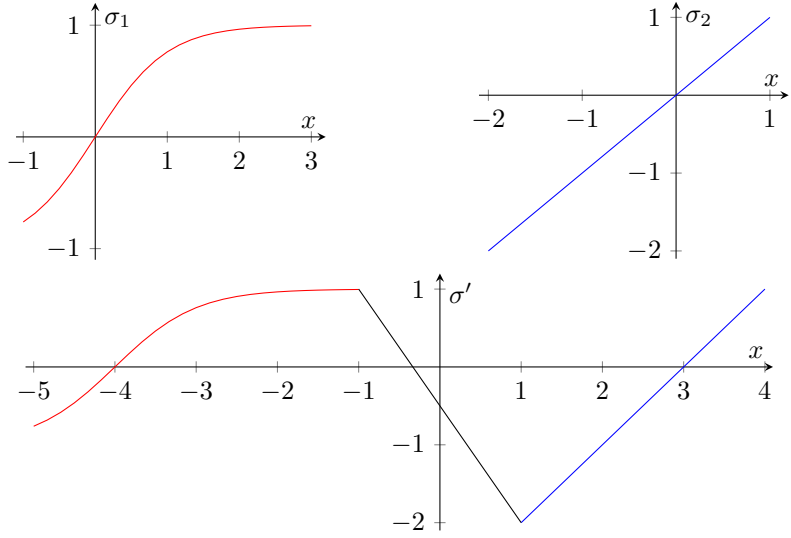

\begin{lemma} \label{lem-conc-cont}

	Let $L$ and $K$ be MPNN layers of type $d_L \to r_L$ and $d_K \to
		r_k$, respectively. Let $X_L \subseteq \Rb^{d_L}$ and $X_K
		\subseteq \Rb^{d_K}$ be bounded, and let $p$ be a natural number.
	There exist two MPNN layers $L'$ and $K'$ such that
	\begin{enumerate}
		\item $L'$ and $K'$ use the same activation function;
		\item $L'$ is equivalent to $L$ over $\Grs_p$ and $X_L$;
		\item $K'$ is equivalent to $K$ over $\Grs_p$ and $X_K$.
	\end{enumerate}

\end{lemma}
\begin{proof}

	Let $L=(W_{1L}, W_{2L}, b_L,  \sigma_L)$ and $K=(W_{1K}, W_{2K},
		b_K,  \sigma_K)$. Let $w_{1,i}$, $w_{2,i}$ and $b_i$ be the
	$i$-th row of $W_{1L}$, $W_{2L}$ and $b_L$ respectively. For
	each $i \in \{ 1, \dots , r_L\}$ and for any $k \in \{1, \dots,
		p\}$ consider the function
	$$ \lambda^{k}_{i}: \Rb^{(k+1)d_L} \to \Rb: (\vec{x_0},
		\vec{x_1}, \dots, \vec{x_k}) \mapsto w_{1, i} \cdot
		\vec{x_0} + w_{2,i} \cdot \vec{x_1} + \dots + w_{2,i}
		\cdot \vec{x_k} + b_i.
	$$
	Then for any $G \in \Grs_p$, any $\chi \in \FeatX{G}{d}{X_L}$,
	and $v \in V(G)$, each component of $L(G)(\chi)(v)$ will
	belong to the image of some function $\lambda^{k}_{i}$ on
	$X_{L}^{k+1}$, with $k$ the degree of $v$. Since
	$X_{L}^{k+1}$ is bounded and $\lambda^{k}_{i}$ is
	continuous, these images are also bounded and their finite
	union over $i \in \{ 1, \dots , r\}$ and $k \in \{1, \dots,
		p\}$ is also bounded. Let $Y_1$ be this union and let $M =
		\max Y_1$.

	For $K$ we can similarly define the functions $\kappa^{k}_{i}$
	and arrive at a bounded set $Y_K \subseteq \Rb$. We then define
	$m=\min Y_2$.

	We will now construct a new activation function $\sigma'$. First
	define the functions $\sigma_L'(x):=\sigma_{L}(x + M_i + 1)$
	for $x \in ]-\infty, -1]$ and $\sigma_K'(x):=\sigma_{K}(x - m_i
		- 1)$ for $x \in [1, \infty[$. Notice how $\sigma_L'$ is simply
	$\sigma_L$ shifted to the left so that its highest possible
	input value, which is $M$, aligns with $-1$. Similarly,
	$\sigma_K'$ is simply $\sigma_K$ shifted to the right so that
	its lowest possible input value, which is $m$, aligns with $1$.
	We then define $\sigma'$ to be any continuous function that
	extends both $\sigma_L'$ and $\sigma_K'$. An example of this
	construction can be seen in
	Figure~\ref{fig:activation_function_merging} with
	$\sigma_1=\tanh$, $M = 3$, $\sigma_2$ the identity, and $m
		= -2$.

	We also construct new bias vectors, obtained by shifting $b_L$ and
	$b_K$ left and right respectively to provide appropriate inputs
	for $\sigma'$. Specifically, we define $b_L' := b_L - (M+1)^{r
				\times 1}$ and $b_K' := b_K + (m+1)^{r \times 1}$.

	Finally, we
	can set $L'=(W_{1L}, W_{2L}, b_L',  \sigma')$ and
	$K'=(W_{1K}, W_{2K}, b_K',  \sigma')$ as desired.
\end{proof}

Thanks to the above lemma, Lemma~\ref{lem-layer-concat} remains
available to concatenate layers.  The part of
Lemma~\ref{lem-layer-concat} that deals with parallel composition
(which is needed to prove closure under concatenation for
multi-layer MPNNs) must be slightly adapted as follows.  It
follows immediately from Lemma~\ref{lem-conc-cont} above and the
original Lemma~\ref{lem-layer-concat}.

\begin{lemma}
	Let $L$ and $K$ be MPNN layers of type $d_L \to r_L$ and
	$d_K \to r_k$, respectively. Let $X_L \subseteq \Rb^{d_L}$ and
	$X_K \subseteq \Rb^{d_K}$ be bounded, and let $p$ be a natural number.
	Let $X = X_L \times X_K \subseteq \Rb^{d_L + d_K}$.
	There exists an MPNN layer that is equivalent to
	$L \parallel K$ over $\Grs_p$ and $X$.
\end{lemma}
\qed

\newcommand{\setF}{\mathcal{F}}
A slightly stricter version of Theorem~\ref{thmcont} can be proven for all MPLang
expressions that are \emph{addition-free}, i.e., do not use the $+$ operator.
We will generalize the notion of $\sigma$-MPLang expressions and $\sigma$-MPNNs
to $\setF$-MPLang and $\setF$-MPNNs, for a set of continuous functions $\setF$.
Indeed, the following proof follows directly from the proof of Theorem~\ref{thmrelu}.

\begin{proposition}
	Any addition-free $\GNNLang$ expression $e$ using the functions $\setF$
	has an equivalent $\mathcal{G}$-$\GNN$ with $\mathcal{G} = \setF \cup \{\identity\}$.
	\label{prop-mplang-no-plus-mpnn}
\end{proposition}

Additionally, if we neither allow the $\Diamon$ operator
(called a \emph{summation-free} expression),
we get an even stricter version of the result

\begin{proposition}
	Any addition-free, summation-free $\GNNLang$ expression $e$ using the functions $\setF$
	has an equivalent $\setF$-MPNN.
	\label{prop-mplang-no-plus-diamon-mpnn}
\end{proposition}
\begin{proof}
	By induction on the structure of $e$, constructing for each $e$ an equivalent MPNN $E$.
	For the base cases we refer to the proof of Theorem~\ref{thmrelu}.
	In the inductive cases we $e$ is of the form $ae_1$ or $f(e_1)$.
	By induction, we have a $\setF$-MPNN $E_1$ that is equivalent to $e_1$
	and let $L = (W_1, W_2, b, \sigma)$ be the last layer of $E_1$.

	If $e$ is of the form $a e_1$ and $\sigma$ is the identity function,
	$E$ is obtained from $E_1$ by replacing the last layer by $(aW_1, aW_2, ab, \identity)$.
	If $\sigma$ is not the identity, we set $E = E_1;(a,0,0,\identity)$.

	If $e$ is of the form $f(e_1)$ and $\sigma$ is the identity,
	$E$ is obtained from $E_1$ by replacing the last layer by $(W_1, W_2, b, f)$.
	If $\sigma$ is not the identity, we set $E = E_1;(1,0,0,\identity)$.
\end{proof}

\section{Approximation by ReLU-MPNNs} \label{secapprox}

Theorem~\ref{thmcont} allows the use of arbitrary activation
functions in the MPNN simulating an MPLang expression; these
activation functions may even be different from the ones applied
in the expression (see the proof of Lemma~\ref{lem-conc-cont}).
What if we insist on MPNNs using a fixed activation function?  In
this case we can still recover our result, if we allow
approximation.  Moreover, we must slightly strengthen our
assumption of feature vectors coming from a bounded domain, to
coming from a compact domain.\footnote{A subset of $\Rb$ or
	$\Rb^d$ is called compact if it is bounded and closed in the
	ordinary topology.}

We will rely on a classical result
in the approximation theory of neural networks
\cite{leshno-approx-nn,pinkus-survey}.\footnote{The stated
	Density Property actually holds not just for ReLU, but for any
	nonpolynomial continuous function.}  In order to
recall this result, we recall that the uniform distance between two
continuous functions $g$ and $h$ from $\Rb$ to $\Rb$ on a compact
domain $Y$ equals $\rho_Y(g,h)=\sup_{x\in Y}|g(x)-h(x)|$.

\begin{densprop}
	Let $Y$ be a compact subset of $\Rb$, let $f: \Rb \to \Rb$ be
	continuous on $Y$, and let $\epsilon > 0$ be a real
	number. There exists a positive integer $n$ and
	real coefficients $a_i,b_i,c_i$, for $i=1,\dots,n$, such that
	$\rho_Y(f, f') \leq \epsilon$, where $f'(x)=\sum_{i = 1}^{n} c_i
		\ReLU (a_i x - b_i)$.
\end{densprop}

We want to extend the notion of uniform distance to GFMTs expressed
in MPLang.  For any MPLang expression $e$ appropriate for input
arity $d$, any class $\G$ of graphs, and any subset
$X \subseteq \Rb^d$, the \emph{image} of $e$ over $\G$ and $X$
is defined as the set $$ \{e(G)(\chi)(v) : G \in \G \ \& \
	\chi \in \FeatX GdX \ \& \ v \in V(G)\}. $$
It is a subset of $\Rb$.  We observe:
\begin{lemma} \label{lem-compact}
	For any natural number $p$ and compact $X$, the image of $e$ over
	$\Grs_p$ and $X$ is contained in a compact set.
\end{lemma}
\begin{proof}
	By induction on the structure of a $e$. For the inductive cases
	we assume the images of $e_1$ and $e_2$ to be contained in the
	compact sets $Y_1, Y_2 \subset \Rb$ respectively.

	\begin{itemize}
		\item If $e$ is of the form $1$, the image of $e$ is
			$\{ 1 \}$ which is a compact subset of $\Rb$.
		\item If $e$ is of the form $P_i$, the image of $e$ is
			the $i$-th projection of $X$ which is compact.
		\item If $e$ is of the form $a e_1$, the image of $e$ is contained
			in $\{ ay \mid y \in Y \}$, which is closed and bounded.
		\item If $e$ is of the form $e_1 + e_2$, the image of $e$
			is contained in $\{ y_1 + y_2 \mid y_1 \in Y_1 \text{ and } y_2 \in Y_2 \}$,
			which is closed and bounded.
		\item If $e$ is of the form $f(e_1)$, the image of $e$ is
			contained in $f(Y_1)$. Since $f$ is continuous, $f(Y_1)$ is
			also a compact subset of $\Rb$.
		\item If $e$ is of the form $\Diamon(e_1)$, by the degree bound $p$,
			the image of $e$ is contained in $\{ y_1 + \dots + y_p \mid y_1, \dots, y_p \in Y_1 \cup \{0\}\}$,
			which is compact.
	\end{itemize}
\end{proof}

With $p$ and $X$ as in the lemma, and
any two MPLang expression $e_1$ and $e_2$ appropriate for input
arity $d$, the set $$ \{
	|e_1(G)(\chi)(v) - e_2(G)(\chi)(v)| : G \in \Grs_p \ \& \
	\chi \in \FeatX GdX \ \& \ v \in V(G)\} $$
has a supremum.  We define $\dist(e_1,e_2)$, the uniform distance between
$e_1$ and $e_2$ over $\Grs_p$ and $X$, to be that supremum.

The main result of this section can now be stated as follows.
Note that we approximate MPLang expressions by $\ReLU$-MPLang
expressions. These can then be further converted to $\ReLU$-MPNNs
by Theorem~\ref{thmrelu}.

\begin{theorem}
	\label{thmapprox}
	Let $p$ and $d$ be natural numbers, and let $X \subset
		\Rb^{d}$ be compact. Let
	$e$ be an MPLang expression appropriate for $d$, and let
	$\epsilon > 0$ be a real number.  There exists a ReLU-MPLang
	expression $e'$ such that $\dist(e,e') \leq \epsilon$.
\end{theorem}
\begin{proof}

	By induction on the structure of $e$. If $e$ is $1$ or of
	the form $P_i$, then $e'$ is simply $e$. In
	the inductive cases where $e$ is of the form $ae_1$, $e_1+e_2$,
	or $f(e_1)$, we consider any $G \in \Grs_p$, any $\chi \in
		\FeatX GdX$, and any $v \in V(G)$, but abbreviate
	$e(G)(\chi)(v)$ simply as $e$.

	Let $e$ be of the form $a e_1$.  If $a=0$ we set $e'=0$.
	Otherwise, let $e_1'$ be the expression obtained by
	induction applied to $e_1$ and $\epsilon/a$.  We then set $e' =
		ae'_1$.  The inequality $|e-e'|\leq \epsilon$ is readily
	verified.

	Let $e$ be of the form $e_1 + e_2$.  For $j=1,2$, let $e'_j$ be
	the expression obtained by induction applied to $e_j$ and
	$\epsilon/2$.  We then set $e'=e'_1 + e'_2$.  The inequality
	$|e-e'|\leq \epsilon$ now follows from the triangle inequality.

	Let $e$ be of the form $f(e_1)$. By Lemma~\ref{lem-compact},
	the image of $e_1$ is a compact set $Y_1 \subset \Rb$. We
	define the closed interval $Y = [\min(Y_1) - \epsilon/2,
		\max(Y_1) + \epsilon/2]$.  By the Density Property, there
	exists $f'$ such that $\rho_Y(f,
		f') \leq \epsilon/2$.  Since $Y$ is compact, $f'$ is uniformly
	continuous on $Y$. Thus there exists
	$\delta>0$ such that $|f'(x) - f'(x')| < \epsilon/2$
	whenever $|x - x'|<\delta$.

	We now take $e_1'$ to be the expression obtained by induction applied to
	$e_1$ and $\min(\delta,\epsilon/2)$.  We see that the image of
	$e_1'$ is contained in $Y$.  Setting $e' = f(e'_1)$,
	we verify that $|e-e'|
		= |f(e_1) - f'(e'_1)| + |f'(e_1)-f'(e'_1)| \leq \epsilon$ as
	desired.

	Our final inductive case is when $e$ is of the form $\Diamon
		e_1$.  We again consider any $G \in \Grs_p$, any $\chi \in
		\FeatX GdX$, and any $v \in V(G)$, but this time abbreviate
	$e(G)(\chi)(v)$ as $e(v)$.
	Let $e_1'$ be the expression obtained by induction applied to
	$e_1$ and $\epsilon/p$. Setting $e' = \Diamon e'_1$, we
	verify, as desired:
	\begin{align*}
		|e(v)-e'(v)| & = |\sum_{u \in \Neigh Gv} e_1(u) - \sum_{u \in
		\Neigh Gv} e_1'(u)|                                           \\
		             & \leq \sum_{u \in \Neigh Gv} |e_1(u)-e_1'(u)|   \\
		             & \leq p (\epsilon/p)                            \\
		             & = \epsilon.
	\end{align*}
	The penultimate step clearly uses that $G$ has degree bound $p$.
	(This degree bound is also used in Lemma~\ref{lem-compact}.)
\end{proof}

\section{Concluding remarks} \label{seconc}

We believe that our approach has the advantage of modularity.  For
example, Theorem~\ref{thmrelu} is stated for ReLU, but holds for
any activation function for which Lemmas
\ref{lem-layer-concat} and \ref{lem-idlayer} can be shown.
We already noted that the Density Property holds not just for
ReLU but for any
nonpolynomial continuous activation function.  It follows that
for any activation function $\sigma$ for which Lemmas
\ref{lem-layer-concat} and \ref{lem-idlayer} can be shown, every
MPLang expression can be approximated by a $\sigma$-MPNN.

The proof of Theorem~\ref{thmcont}, and the Propositions~\ref{prop-mplang-no-plus-mpnn}
and \ref{prop-mplang-no-plus-diamon-mpnn} give us a set of sufficient conditions
such that for each $\setF$-MPLang expression,
there is an equivalent $\setF$-MPNN.
The first condition is that Lemma~\ref{lem-conc-cont} is true when
the activation functions are restricted to $\setF$ and without the restrictions
on the graph and the feature map.
The second condition is Lemma~\ref{lem-idlayer} holds for
$\setF$ instead of $\ReLU$.

It would be interesting to see if this set of requirements for
an exact translation can be further refined or if it can be
proven that this is a set of necessary conditions.

We have so far proven 2 sets of functions such that their
MPLang expressions have equivalent MPNNs.
First there is the set $\{ \ReLU \}$ and using
Lemma~\ref{lem-idlayer} we can prove that the
same holds for $\{ \ReLU, \identity \}$.
Second there is the set of all continuous
functions under the restriction that all
graphs are of a certain bounded degree $p$
and that all feature vectors come from some
compact set. It would be interesting to see
if there are other sets of functions $\setF$
for which each $\setF$-MPLang expression has an equivalent
$\setF$-MPNN and sets for which this is not the case.

It would be interesting to see counterexamples that show that
Theorems \ref{thmcont} and \ref{thmapprox} do not hold without
the restriction to bounded-degree graphs, or to features from a
bounded or compact domain.  Such counterexamples can probably be
derived from known counterexamples in analysis or approximation
theory.

Finally, in this work we have focused on the question whether MPLang
can be simulated by MPNNs.  However, it is also interesting to
investigate the expressive power of MPLang by itself.  For
example, is the GFMT $T_3$ from Example~\ref{exgfmt} expressible
in MPLang?


\begin{thebibliography}{10}
	\providecommand{\url}[1]{\texttt{#1}}
	\providecommand{\urlprefix}{URL }
	\providecommand{\doi}[1]{https://doi.org/#1}

	\bibitem{grohe-random-gnn}
	Abboud, R., Ceylan, I., Grohe, M., Lukasiewicz, T.: The surprising power of
	graph neural networks with random node initialization. In: Zhou, Z.H. (ed.)
	Proceedings 30th International Joint Conference on Artificial Intelligence.
	pp. 2112--2118. ijcai.org (2021)

	\bibitem{ahv_book}
	Abiteboul, S., Hull, R., Vianu, V.: Foundations of Databases. Addison-Wesley
	(1995)

	\bibitem{arora-relu}
	Arora, R., Basu, A., Mianjy, P., Mukherjee, A.: Understanding deep neural
	networks with rectified linear units. In: Proceedings 6th International
	Conference on Learning Representations. OpenReview.net (2018)

	\bibitem{azizian}
	Azizian, W., Lelarge, M.: Expressive power of invariant and equivariant graph
	neural networks. In: Proceedings 9th International Conference on Learning
	Representations. OpenReview.net (2021)

	\bibitem{balcilar}
	Balcilar, M., H\'eroux, P., et~al.: Breaking the limits of message passing
	graph neural networks. In: Meila, M., Zhang, T. (eds.) Proceedings 38th
	International Conference on Machine Learning. Proceedings of Machine Learning
	Research, vol.~139, pp. 599--608 (2021)

	\bibitem{barcelo-gnn-sigrec}
	Barcel\'o, P., Kostylev, E., Monet, M., P\'erez, J., Reutter, J., Silva, J.:
	The expressive power of graph neural networks as a query language. SIGMOD
	Record  \textbf{49}(2),  6--17 (2020)

	\bibitem{barcelo-log-expr-gnn}
	Barcel\'o, P., Kostylev, E., Monet, M., P\'erez, J., Reutter, J., Silva, J.:
	The logical expressiveness of graph neural networks. In: Proceedings 8th
	International Conference on Learning Representations. OpenReview.net (2020)

	\bibitem{matlang_tods}
	Brijder, R., Geerts, F., Van~den Bussche, J., Weerwag, T.: On the expressive
	power of query languages for matrices. ACM Transactions on Database Systems
	\textbf{44}(4),  15:1--15:31 (2019)

	\bibitem{graedel-prov-modal}
	Dannert, K., Gr\"adel, E.: Semiring provenance for guarded logics. In:
	Madar\'asz, J., Sz\'ekely, G. (eds.) Hajnal Andr\'eka and Istv\'an N\'emeti
	on the Unity of Science, Outstanding Contributions to Logic, vol.~19, pp.
	55--79. Springer (2021)

	\bibitem{floris_lagraphs}
	Geerts, F.: On the expressive power of linear algebra on graphs. Theory of
	Computing Systems  \textbf{65}(1),  179--239 (2021)

	\bibitem{gr-gnn-tl}
	Geerts, F., Reutter, J.: Expressiveness and approximation properties of graph
	neural networks. In: ICLR. OpenReview.net (2022), to appear

	\bibitem{gilmer-mpnn}
	Gilmer, J., Schoenholz, S., et~al.: Neural message passing for quantum
	chemistry. In: Precup, D., Teh, Y. (eds.) Proceedings 34th International
	Conference on Machine Learning. Proceedings of Machine Learning Research,
	vol.~70, pp. 1263--1272 (2017)

	\bibitem{goodfellow-book}
	Goodfellow, I., Bengio, Y., Courville, A.: Deep Learning. MIT Press (2016)

	\bibitem{provsemirings}
	Green, T., Karvounarakis, G., Tannen, V.: Provenance semirings. In: Proceedings
	26th ACM Symposium on Principles of Database Systems. pp. 31--40 (2007)

	\bibitem{grohe-logic-gnn}
	Grohe, M.: The logic of graph neural networks. In: Proceedings 36th Annual
	ACM/IEEE Symposium on Logic in Computer Science. pp. 1--17. IEEE (2021)

	\bibitem{gurevich_algfeas}
	Gurevich, Y.: Algebras of feasible functions. In: Proceedings 24th Symposium on
	Foundations of Computer Science. pp. 210--214. IEEE Computer Society (1983)

	\bibitem{gurevich_challenge}
	Gurevich, Y.: Logic and the challenge of computer science. In: B\"orger, E.
	(ed.) Current Trends in Theoretical Computer Science, pp. 1--57. Computer
	Science Press (1988)

	\bibitem{hamilton-grl}
	Hamilton, W.: Graph Representation Learning. Synthesis Lectures on Artificial
	Intelligence and Machine Learning, Morgan {\&} Claypool (2020)

	\bibitem{hella-distributed}
	Hella, L., J\"arvisalo, M., Kuustisto, A., Laurinharju, J., Lempi\"ainen, T.,
	Luosto, K., Suomela, J., Virtema, J.: Weak models of distributed computing,
	with connections to modal logic. Distributed Computing  \textbf{28},  31--53
	(2015)

	\bibitem{leshno-approx-nn}
	Leshno, M., Lin, V., Pinkus, A., Schocken, S.: Multilayer feedforward networks
	with a nonpolynomial activation function can approximate any function. Neural
	Networks  \textbf{6}(6),  861--867 (1993)

	\bibitem{morris-wl-survey}
	Morris, C., et~al.: Weisfeiler and {L}eman go machine learning: The story so
	far. arXiv:2122.09992 (2021)

	\bibitem{pinkus-survey}
	Pinkus, A.: Approximation theory of the {MLP} model in neural networks. Acta
	Numerica  \textbf{8},  143--195 (1999)

	\bibitem{rudin}
	Rudin, W.: Principles of Mathematical Analysis. McGraw Hill, third edn. (1976)

	\bibitem{aima-book}
	Russell, S., Norvig, P.: Artificial Intelligence: A Modern Approach. Pearson,
	fourth edn. (2022)

	\bibitem{understanding-ml-book}
	Shalev-Shwartz, S., Ben-David, S.: Understanding Machine Learning: From Theory
	to Algorithms. Cambridge University Press (2014)

\end{thebibliography}
\end{document}